\documentclass[journal,oneside,web]{ieeecolor}
\usepackage{tmi}
\usepackage[noadjust]{cite}
\usepackage{amsmath,amssymb,amsfonts,mathtools,commath}
\usepackage{algorithm,algorithmic}
\usepackage{graphicx}
\usepackage{textcomp}
\usepackage{pgfplots}
\usepackage{balance}
\usepackage{multicol}
\usepackage{booktabs}
\usepackage{tabularx}
\usepackage{multirow}
\usepackage{comment}
\usepackage{todonotes}
\usepackage[normalem]{ulem}

\usepackage[caption=false,font=normalsize,labelfont=sf,textfont=sf]{subfig}
\graphicspath{{./images/}}

\def\BibTeX{{\rm B\kern-.05em{\sc i\kern-.025em b}\kern-.08em
		T\kern-.1667em\lower.7ex\hbox{E}\kern-.125emX}}
\pgfplotsset{compat=newest}
\usetikzlibrary{positioning,spy,backgrounds,calc,arrows}
\DeclareMathOperator*{\argmin}{arg\,min}

\newtheorem{lemma}{Lemma}
\newcommand{\bs}{\boldsymbol}
\newcommand{\mbf}{\mathbf}
\newcommand{\A}{\mathcal A_{\bs s}}
\newcommand{\rec}{\bs{\widehat \rho}}
\newcommand{\Es}{\mathbb{E}_{\bs s } }
\newcommand{\Eus}{\mathbb E_{\bs u}}

\begin{document}
\bstctlcite{IEEEexample:BSTcontrol}

\title{ENSURE: A general approach for unsupervised training of deep image reconstruction algorithms}
\author{Hemant~Kumar~Aggarwal,~\IEEEmembership{Member,~IEEE,}
	    Aniket Pramanik,~\IEEEmembership{Student Member,~IEEE,}
        Maneesh John,
        Mathews~Jacob,~\IEEEmembership{Fellow,~IEEE}%
\thanks{Hemant~Kumar~Aggarwal (email: jnu.hemant@gmail.com), Aniket Pramanik (email: aniket-pramanik@uiowa.edu), Maneesh John (email: maneesh-john@uiowa.edu), and Mathews Jacob (email: mathews-jacob@uiowa.edu) are with the Department of Electrical and Computer Engineering, University of Iowa, Iowa City, IA, USA, 52242.
}
\thanks{Manuscript received Month day, year; revised Month day, year.}
\thanks{This work is supported by 1R01EB019961-01A1 and 1R01AG067078-01A1. This work was conducted on an MRI instrument funded by 1S10OD025025-01. }
}

\maketitle

\begin{abstract}
Image reconstruction using deep learning algorithms offers improved reconstruction quality and lower reconstruction time than classical compressed sensing and model-based algorithms. Unfortunately, clean and fully sampled ground-truth data to train the deep networks is often unavailable in several applications, restricting  the applicability of the above methods. We introduce a novel metric termed the ENsemble Stein's Unbiased Risk Estimate (ENSURE) framework, which can be used to train deep image reconstruction algorithms without fully sampled and noise-free images. The proposed framework is the generalization of the classical SURE and GSURE formulation to the setting where the images are sampled by different measurement operators, chosen randomly from a set. We evaluate the expectation of the GSURE loss functions over the sampling patterns to obtain the ENSURE loss function. We show that this loss is an unbiased estimate for the true mean-square error, which offers a better alternative to GSURE, which only offers an unbiased estimate for the projected error. Our experiments show that the networks trained with this loss function can offer reconstructions comparable to the supervised setting. While we demonstrate this framework in the context of MR image recovery, the ENSURE framework is generally applicable to arbitrary inverse problems. 
\end{abstract}


\section{Introduction}
The recovery of images from a few of their noisy measurements is a classical inverse problem, which is central to several imaging modalities. For instance, the recovery of MR images from a few of their multi-channel k-space samples using compressed sensing~(CS)~\cite{lustig2008compressed,candes2007sparsity,fessler2010magazine} algorithms is a popular approach to speed up the scans. Deep learning methods have recently emerged as powerful alternatives to CS-based approaches.  These methods rely on convolutional neural network (CNN) models that offer reduced computational complexity and improved performance. These deep learning methods can be categorized as direct inversion schemes~\cite{wangCTtmi2017,jong2019kspace}and model-based schemes~\cite{modl,admmnet,adler2018learned,hammernik,zhang2017magazine,dagan}. The model-based methods account for the imaging physics using a numerical model along with a deep, learned prior. The CNN modules in the above algorithms have been trained in a supervised mode, where an extensive amount of fully sampled training data is required. While large multi-center datasets are emerging for common applications such as brain and knee imaging \cite{fastMRI}, similar datasets are not available for all settings. More importantly, it is often challenging to acquire fully sampled datasets in applications including dynamic and static MRI with high spatial and temporal resolution. Unsupervised methods that can learn the deep networks from undersampled data directly can significantly improve the applicability of these approaches.

The unsupervised optimization of the parameters of image reconstruction and denoising algorithms have a long history. Early approaches relied on L-curve~\cite{lcurve} and generalized cross-validation~\cite{gcv} to optimize the regularization parameters in regularized image recovery. Another approach is to use Stein's Unbiased Risk Estimate (SURE)~\cite{sure} to determine the optimal parameters. The SURE loss is an unbiased estimate of the mean-square-error~(MSE) that only depends on the recovered image and the noisy measurements. For example, SURE-let- and PURE-let-based methods optimize the thresholds in wavelet-based image denoising with Gaussian \cite{sureSink,surelet} and Poisson noise\cite{purelet}. Recently, SURE has been used to train deep image denoisers in \cite{ldampSURE,koreanDenoisingNIPS,koreanReconCVPR2019}. The SURE approach was extended to inverse problems with a rank-deficient measurement operator, which is termed generalized SURE~(GSURE)~\cite{eldarGSURE}. GSURE provides an unbiased estimate of the projected MSE, which is the expected error of the projections in the range space of the measurement operator. GSURE was recently used for inverse problems in \cite{ldampSURE}. Unfortunately, the experiments in \cite{ldampSURE} show that the GSURE-based projected MSE is a poor approximation of the actual MSE in the highly undersampled setting. To improve the performance, the authors trained the denoisers at each iteration in a message-passing algorithm in a layer-by-layer fashion using classical SURE, which is termed LDAMP-SURE \cite{ldampSURE}. A similar approach was used for MRI recovery in \cite{koreanReconCVPR2019}. The GSURE metric was also used in a deep image prior framework recently \cite{deepGSURE}. In the supervised context, the end-to-end training of model-based algorithms \cite{modl,admmnet,adler2018learned,hammernik,zhang2017magazine,dagan} has been shown to be significantly more efficient than the layer-by-layer training strategies \cite{ldampSURE,koreanReconCVPR2019}.

 Unsupervised strategies that do not use SURE have also been recently  introduced. For instance,  Noise2Noise~\cite{noise2noise} relies on a pair of noisy images to train a denoiser without the need for clean images. Noise2Void~\cite{noise2void} is a blind spot method that excludes the central pixel from the network's receptive field. 
 A similar strategy was used in self-supervised learning using data undersampling (SSDU) \cite{noise2void} for the end-to-end training of unrolled algorithms. This approach partitions the acquired measurements into two disjoint sets.  The first set is used for the data consistency and the remaining set is used to define the loss. These methods offer reasonable results in the denoising (full-sampled) and low undersampling setting. However, as the undersampling factor increases, the available measurements are limited, restricting performance.

We introduce a novel loss function termed ensemble Stein's Unbiased Risk Estimate (ENSURE) to train deep networks. We assume that the sampling operator for each image is randomly chosen from a set; we evaluate the expectation of the GSURE losses over the sampling patterns. We show that the resulting loss metric is an unbiased estimate of the true MSE, hence, is a superior loss function when compared to projected SURE \cite{eldarGSURE}. We note that such a measurement scheme can be realized in practice, where a different sampling mask  can be chosen from a lookup table for each image. Similar to classical SURE metrics \cite{sureSink,eldarGSURE}, the proposed ENSURE loss metric is the sum of data consistency and divergence terms. While the divergence term is similar to GSURE, the data consistency term in ENSURE is the sum of the weighted projected losses \cite{eldarGSURE}, where the weighting depends on the class of sampling operators. When the divergence term is not used, the proposed loss reduces to only the projection error, which is observed to offer poor performance. The divergence term serves as a network regularization that minimizes the over-fitting to the measurement noise. The experiments demonstrate that the proposed ENSURE approach can offer results comparable to supervised training. 

The main focus of this paper is to introduce the ENSURE metric for unsupervised training of deep networks. The main contributions of the paper are as follows
\begin{itemize}
	\item We propose to sample different images using different undersampling operators. We introduce an new metric involving the weighted average of projected mean-square errors. We mathematically show that the expectation of this metric is essentially the true mean-square error. 
 \item When fully sampled images are not available, we show that the weighted projected MSE can be approximated using the ENSURE metric, thus eliminating the need for fully sampled training datasets. 
 \item We rigorously validate the proposed ENSURE metric in multiple datasets, undersampling patterns, and with different deep image reconstruction algorithms. We also compare the proposed framework against state-of-the-art unsupervised methods.
\end{itemize}

\section{Background}
In this section, we briefly review the background to make the paper self-contained.

\subsection{Inverse Problems \& Supervised Learning}
We consider the acquisition of the complex-valued image  $\bs{\rho} \in \mathbb{C}^N$ 
using the multichannel measurement operator $\mathcal A_{\bs s}$ as \begin{equation}
\label{fwd}
\bs y_{\bs s} = \mathcal A_{\bs s}~\bs{\rho} + \bs{n}.
\end{equation} 
 We assume the noise~$\bs{n}$ to be complex-valued Gaussian distributed with zero mean and covariance matrix $\mbf C$ such that $\bs{n}\sim  \mathcal N(0,\mbf C)$.  We note that the probability distribution function~(PDF) of $\bs{y_s}$ is given by $
p(\bs{y_{\bs s}}|\bs s) = \mathcal{N}\left(\bs A_{\bs s}\bs{\rho}, \mbf C\right)$. 

In this work, we assume the measurement operator $\mathcal A_{\bs s}$ to be randomly chosen and parameterized by a random vector $\bs s$  that is independent of $\bs n$, drawn from a set $\mathcal S$. For example, $\mathcal A_{\bs s}$ can be viewed as a Fourier transform followed by multiplication with a k-space sampling mask parameterized by $\bs s$ in the single-channel MRI setting. The binary mask $\bs s$ is chosen at random, where the probability of a specific value being one is pre-specified. In the multi-channel setting, $\mathcal A_s$ corresponds to multiplying $\bs \rho$ with the coil sensitivities to obtain multiple sensitivity weighted images, followed by the undersampled Fourier transform of each of the images. The goal of the image reconstruction scheme is to recover $\bs\rho$ from $\bs y_{\bs s}$. 
The inverse problem in~\eqref{fwd} gets simplified to the image denoising model when $\A = \mathcal I$. 

Supervised deep learning methods learn to recover the fully sampled image $\bs{\widehat \rho}$ only from noisy and undersampled measurements $\bs{y_s}$. Let 
\begin{equation}\label{u}
\bs{u}=\A ^H \mathbf C^{-1}\bs{y_s}.
\end{equation}
We note that $\bs u$ is a random variable because of the additive noise in \eqref{fwd}.  In that case, the recovery using a deep neural network $f_\Phi$ with trainable parameters $\Phi$ is represented as
\begin{equation}
\label{recon}
\boldsymbol{\widehat \rho} = f_{\Phi}(\boldsymbol {u}). 
\end{equation}
Here $f_\Phi$ can be a direct inversion or a model-based deep neural network.
Supervised approaches often rely on the loss function:
\begin{equation}
\label{sup_mse}
\text{Sup-Loss}= \sum_{i=1}^{N_{\rm train}}\| \bs{\widehat \rho}_i -  \bs \rho_i \|_2^2
\end{equation}
to train the network using $N_{\rm train}$ number of images.

\subsection{Unsupervised Learning of Denoisers}
\label{denoiser}
When noise-free training data is unavailable, denoising approaches such as Noise2Noise~\cite{noise2noise} and blind-spot methods~\cite{noise2void,ssdu,probN2V} have been introduced for unsupervised learning of the network parameters $\Phi$. Assuming the additive noise $\bs n$ to be Gaussian distributed, the SURE~\cite{sure} approach uses the loss function
\begin{equation}
\label{sure}
{\rm SURE}(\bs{\widehat \rho},\bs{u}) = \| \rec- \bs {u} \|^2_2  + 2  \sigma^2 \nabla_{\bs u} \cdot f_\Phi(\bs{u}) - N\sigma^2, 
\end{equation}
which is an unbiased estimate of the true mean-square error, denoted by 
\begin{equation}\label{mse}
{\rm MSE} = \mathbb E_{\bs u} ~\|\widehat{\bs \rho} - \bs \rho\|^2
\end{equation}
Note that the expression in \eqref{sure} does not depend on the noise-free images $\bs \rho$; it only depends on the noisy images $\bs u$ and the estimates $\bs{\widehat \rho}$. In \eqref{sure}, $\nabla_{\bs u} \cdot f_\Phi(\bs{u})$ represents network divergence, which is often estimated using Monte Carlo simulations~\cite{mcsure}. Several researchers have adapted SURE as a loss function for the unsupervised training of deep image denoisers~\cite{ldampSURE,koreanReconCVPR2019} with performance approaching that of supervised methods.

\subsection{Unsupervised training with rank deficient $\mathcal A$}
The unsupervised training of deep networks for image recovery is significantly more challenging than the denoising setting when $\mathcal A$ is rank deficient. In particular, the measurement model only acquires partial information about the images, in addition to the additive noise. The Deep Image Prior (DIP)~\cite{dip2018} approach exploits the structural bias of CNNs towards natural images; they optimize the network parameters such that the loss $\| \bs A_{\bs s}\bs{\widehat \rho} - 	\bs y_{\bs s} \|_2^2$ is minimized. Since this approach requires the network to be trained for each image, the computational complexity is high during inference. In addition, the quality of the recovered images is often inferior to the ones obtained from supervised training. Another challenge is the need for manual early stopping of the algorithm to minimize the over-fitting to noise. 

One may use the generalization of a DIP approach for an ensemble of images, which we term as measurement domain (k-space in MRI) MSE loss:
\begin{equation}
	\label{kmse}
	\text{K-MSE}= \sum_{i=1}^{N_{\rm train}} \| \mathcal A_{{\bs s}_i} \rec_i -  \bs{y_s}_i \|_2^2.
\end{equation}
This K-MSE  approach is prone to overfitting the reconstructed image to the noisy measurements~\cite{ssdu}.  The SSDU approach \cite{ssdu} was introduced to minimize this over-fitting ~\cite{ssdu}. SSDU, which may be viewed as an extension of blind-spot methods (e.g., \cite{noise2void}), suggests partitioning the available k-space into two disjoint groups for the data-consistency step and loss function estimation, respectively. The partitioning of k-space, along with the use of different sampling masks for different images, is observed to provide improved results when compared to~\eqref{kmse}. 

The projected GSURE approach was used to train model-based deep learning algorithms in the unsupervised setting in \cite{ldampSURE}. Specifically, the GSURE loss is an unbiased estimate of the projected MSE
\begin{equation}
	\label{pmse}
	{\rm MSE}_{\bs s} = \mathbb E_{{\boldsymbol \rho}\sim \mathcal M} \left\|  \mbf P_{\bs s}\left(\rec - \bs \rho\right)\right\|^2,
\end{equation}
where $\mbf P_{\bs s}=\A^H (\A \A^H) ^{-1} \A $ is the projection operator to the range space of $\A^H$ and $\mathcal M$ is the set of images. The authors of \cite{ldampSURE} noted that $ {\rm MSE}_s$ is a poor approximation of ${\rm MSE}$ in the highly undersampled setting. Hence, instead of directly using GSURE, the LDAMP algorithm instead relies on layer-by-layer training of deep learned denoisers \cite{ldampSURE} using the SURE loss discussed in Section \ref{denoiser}. This scheme implicitly assumes that the alias artifacts at each iteration are Gaussian distributed, which is not a realistic assumption. In addition, the  end-to-end training of model-based algorithms \cite{modl,admmnet,adler2018learned,hammernik} often offers improved performance compared to the above layer-by-layer strategies. 

\section{Ensemble SURE (ENSURE) framework}
The training of the deep network $f_{\Phi}$ in \eqref{recon} using only the measurements $\bs y_s$ in \eqref{fwd} is challenging when the sampling operator $\mathcal A_{\bs s}$ is rank deficient. More specifically, $\bs y_s$ carries only partial information about the images $\bs \rho$, in addition to being noisy. To improve the training in the setting without fully sampled images, we consider the sampling of each image using a different sampling pattern. In the MRI context, we assume the k-space sampling mask $\bs s$ to be a random vector drawn from the distribution $S$. We note that this acquisition scheme is realistic and can be implemented in many applications. For instance, it may be difficult to acquire a specific image in a fully sampled fashion due to time constraints. However, one could use a different undersampling mask for each image $\bs \rho \sim \mathcal M$. The main theoretical contribution of this work is that the expectation of the weighted SURE metrics over the sampling patterns is equal to the true MSE, which offers a better metric than GSURE, which approximates the projected MSE.

In the next section, we consider the expectation of the projected MSE over the sampling patterns, which involve the ground truth images. We show that averaging the projected MSE values will not approximate the true MSE, but a weighted version of the MSE. We thus define a projected and weighted error, which when averaged over the sampling patterns will approximate the true MSE. In Section \ref{ensuresection}, we introduce the ENSURE metric, which evaluates the projected and weighted error metric without using the ground truth images; we show that the expectation of ENSURE yields the true MSE. 
\subsection{Averaging the projected error over sampling patterns}
 We consider a specific image $\rho$ and consider the projected error $\|\mbf P_{\bs s}\rec -\mbf P_{\bs s} \bs{\rho}\| ^2$. In the noiseless setting, we  compute $\mbf P_{\bs s} \bs{\rho}$, which depends on the sampling pattern $\bs s$. In a noisy setting, this measure cannot be directly computed. We now compute the expectation of the projected errors over an ensemble of measurement operators: 
\begin{equation}
	\label{qmse}
	\mathcal Q(\bs u) = \mathbb{E}_{\bs u}\left[\mathbb{E}_{\bs s \sim \mathcal S}\left[ \|\mbf P_{\bs s}\left(\rec -\bs{\rho}\right)\| ^2 \right]\right].
\end{equation}
The following result shows that the above measure only gives a weighted version of the true MSE in \eqref{mse}. 
\begin{lemma}
	\label{lemma1}
The expected loss $\mathcal Q$ in \eqref{qmse} is equal to the weighted MSE:
	\begin{equation}
		\label{wmse}
		\mathcal Q = \mathbb E_{\bs{u}} \left[ \| \mbf W \left( \rec - \bs{\rho} \right)\| ^2 \right], 
	\end{equation}
	where $\mathbf W$ is a weighting matrix that is dependent on the set of measurement operators.
\end{lemma}
\begin{proof}
Denoting the error $\bs e = \rec -\bs{\rho}$ and ignoring its dependence on $\bs s$, we obtain
\begin{subequations}
	\label{proof1}
\begin{align}
	\mathcal Q	&=	\mathbb{E}_{\bs u}\left[\mathbb{E}_{\bs s }\left[ \|\mbf P_{\bs s} \bs e\|^2 \right] \right] \\
	&\approx	\mathbb E_{\bs u}  \left[\bs e^T\Es \left[ \mbf P_{\bs s}\right] \bs e \right] = \mathbb E_{\bs u}  \left[\bs e^T\mbf Q \bs e \right].
\end{align}
We used the symmetry of the projection operators and $\mbf P^2=\mbf P$ in the above step. Here, $\mbf Q$ is the ensemble of projection operator. We denote the matrix square root of $\mbf Q$ by $\mbf W$ (i.e., $\mbf W^2 = \mbf Q$) to obtain 
\begin{align}
\mathcal Q	&=	\Eus \left[\bs e^T\mbf W^T\mbf W \bs e\right]\\
&=	\Eus \left[\| \mbf W \left( \rec -\bs{\rho} \right) \|_2^2 \right], 
\end{align}
\end{subequations}
\end{proof}

Note that $\mathbf Q$ corresponds to the average of the projection operators. Depending on the sampling operators, the weighting by $\mathbf W$ amounts to weighting some aspects of the images more heavily than other features. 

We will now illustrate the above result in the special case of single-channel MRI when $\A=\mbf S\mathbf F\in \mathbb C^{n\times m}$, where $\mathbf F \in \mathbb C^{m\times m}$ is the Fourier matrix and $\mbf S \in \mathbb R^{n\times m}; n<m$ is the sampling matrix. The projection $\mathbf{P_s}$ gets simplified as $\mbf{P_s}= \A^H (\A \A^H)^{-1} \A
= \mathbf F^H \mbf M_{\bs s} \mathbf F$. Here, $\mbf M_{\bs s} =\mbf S^T\mbf S \in \mathbb C^{m\times m}$ is a diagonal matrix. The diagonal entries are one if the specific sample is acquired and zero otherwise. 

We now consider the diagonal entries to be drawn from a Bernoulli distribution with a probability that depends on the spatial frequency location. Note that variable-density masks are routinely used in CS settings. We thus obtain $\mbf Q=\mbf F^H  {\rm diag}(\bs d) \mbf F$, which translates to 
\begin{equation}\label{key}
\mathbf W = \mathbf F^H~{\rm diag}(\sqrt{\bs d}) \mathbf F
\end{equation}

The weighted loss $\mathcal Q$ is equivalent to the MSE if the sampling density is uniform. However, most sampling operators rely on variable-density masks with higher density in the k-space center. In this case, the training of the networks using $\mathcal Q$ will translate to reduced weighting for higher Fourier samples, resulting in reduced high-frequency details. 

\subsection{ENSURE: Unbiased Estimate of MSE}
\label{ensuresection}

To compensate for this weighting, we consider the weighted version of the projected error and compute the expectation over the set of sampling patterns $\mathcal S$ and the images, or equivalently $\bs u$:
\begin{equation}
\label{wpmse}
\mathcal L = \mathbb E_{\bs u} \left [\mathbb E_{\bs s \sim \mathcal S} \left[  \|\mbf R_{\bs s} (\rec -\bs\rho)\|_2^2 \right] \right], 
\end{equation}
where $\mbf R_s = \mbf W^{-1}\mbf P_{\bs s}$ is the weighted projection operator. We note that $\mathbf R_{\bs s} \bs \rho$ still lives in the range space of $\mathcal A$. Using a similar argument in Lemma 1, we can show that $\mathcal L = {\rm MSE}$.

We note that the loss function \eqref{wpmse} depends on the fully sampled ground truth images $\bs \rho_i; i=1,\ldots N$. Hence, it is impossible to directly compute it in an unsupervised setting. We hence use the SURE approach to evaluate it without requiring the ground truth images. We expand \eqref{wpmse} as 
\begin{equation}
    \label{expanded}
	\begin{split}
		\mathcal{L}= \Eus \Es \left[ \|\mbf{ R_s} \rec \|_2^2 \right] + \Eus \Es \left[ \|\mbf{ R_s}\bs\rho\|_2^2 \right]  \\
		- 2\Eus \Es\left[ \rec^T \underbrace{\mbf{ R_s}^T\mbf{ R_s}}_{\mathbf D} \bs\rho \right]			
	\end{split}	
\end{equation}
We note that the estimate $\widehat{\bs \rho} = f_{\Phi}(\bs u)$ in our setting is provided by a deep network $f$, whose weights are denoted by $\Phi$ and whose input is 
\begin{equation}
\label{ueq}
 \bs u = \mathcal A_{\bs s}^H \mathbf C^{-1}\bs y    
\end{equation}
Our goal is to train the network $f_{\Phi}$ or equivalently minimize the above expression with respect to $\widehat{\bs \rho}$.

We note that the first term in \eqref{expanded} is independent of $\rho$ and can be computed. The second term is a constant that does not depend on  $\widehat{\bs \rho}$.  The third term, which includes both $\bs \rho$ and $\widehat{\bs \rho}$. We show in the Appendix that this term can be computed using SURE as 
\begin{equation}
	\label{simplified}
	\begin{split}
		\Eus \Es\left[ \rec^T \mbf{ R_s}^T\mbf{ R_s} \bs\rho \right]	= \Eus \Es \left[ \rec^T\mbf D\bs \rho_{\rm LS} - \nabla_{\bs{u}} \cdot \mbf D \rec \right]
	\end{split}	
\end{equation}
Here, $\bs\rho_{\rm LS}$ is the least-squares solution specified by 
\begin{equation}\label{ls}
\bs\rho_{\rm LS}=\left(\mathcal A_s^H \mathbf C^{-1}\mathcal A_s\right)^{\dagger} \mathcal A_s ^H \mathbf C^{-1} \bs y_{\bs s} .
\end{equation}
 The term $\nabla_{\bs{u}}\cdot \mathbf D f_{\Phi}(\bs u)$ represents the divergence of the network $f_\Phi$ with respect to its input $\bs{u}$.

Combining \eqref{expanded} and \eqref{simplified}, we obtain 
\begin{eqnarray}\nonumber
    	\mathcal{L}&=& \Eus \Es \left[ \|\mbf{ R_s} \rec \|_2^2 + \|\mbf{ R_s}\bs\rho\|_2^2 
		- 2 (\mbf{ R_s}\rec)^T\mbf{ R_s}\bs \rho_{\rm LS}\right]  \\\label{expanded1}
	&&\qquad  + 2~~\Eus \Es\left[\nabla_{\bs{u}} \cdot (\mbf D \rec) \right]\\
		\label{final}
		&=& \Eus \Es \left[ \|\mbf{ R_s} (\rec -\bs \rho_{\rm LS} )  \|_2^2 +  2\nabla_{\bs{u}} \cdot (\mbf D \rec) \right] + \kappa
\end{eqnarray}
We combined the first and third terms in \eqref{expanded} and completed the square to obtain the ENSURE loss as \eqref{final}.

\vspace{5em}
Here, 
\begin{equation}
\kappa= \Eus \Es \left[ \|\mbf{ R_s}\bs\rho\|_2^2 - \|\mbf{ R_s}\rho_{\rm{LS}} \|_2^2\right]
\end{equation}
is a constant that is independent of the network.

In practice, we approximate the expectation by a summation 
\begin{equation}
\label{ensure}
\text{ENSURE}= \frac{1}{N}\sum_{i=1}^{N} \left(\underbrace{\|\mbf{ R_{\bs s_i}} ( \rec_i -\bs \rho_{\rm{LS},i})  \|_2^2}_{\rm data term} + 2\underbrace{\nabla_{\bs{u_i}} \cdot  \left(\mathbf D f_\Phi(\bs{u_i})\right)}_{\rm divergence}\right) 
\end{equation}
where we ignore the constant $\kappa$. Here, $\bs u_i$ are the zero-filled reconstructions (see Section  \eqref{singlechannel} for details) of the $i^{\rm th}$ image. We assume that each image $\rho_i$ is sampled by the corresponding sampling pattern $\bs s_i$.

The loss function in \eqref{ensure} is an unbiased estimate for the true MSE, specified by \eqref{mse}, up to a constant. Since this loss function does not require the fully sampled reference images, it can be used to train a deep learning image reconstruction algorithm in an unsupervised fashion.  The data term in \eqref{ensure} involves the average of the weighted projected losses evaluated over the $N$ training examples. The divergence term is computed over all the data and serves as a regularization term to account for the impact of noise. 

\begin{figure}
	\centering
	\subfloat[Data-term in \eqref{ensure}]{	
		\includegraphics[width=.9\linewidth]{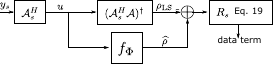}   
	}
	
	\subfloat[Divergence in \eqref{ensure}]{
		\includegraphics[width=.7\linewidth]{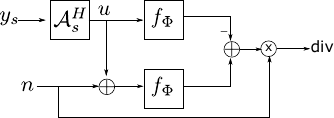}   
	}
	
	\caption{Visual representation of the computation of the data and divergence terms in the proposed ENSURE estimate in~\eqref{ensure}. Here, $\oplus$ and $\otimes$ represent the addition and inner product, respectively. }
	\label{fig:sketch}
\end{figure}

We now focus on some special cases in the context of MRI.
\subsubsection{Single-channel setting with $\mathbf C = \sigma^2 \mathbf I$}
\label{singlechannel}
 
 In many cases, we assume that $\mathbf C = \sigma^2 \mathbf I$, when we obtain 
\begin{equation}\label{ls}
\bs\rho_{LS}=\left(\mathcal A_s^H \mathcal A_s\right)^{\dagger} \mathcal A_s ^H \bs y_{\bs s} =  \mathcal A_s ^H \left(\mathcal A_s \mathcal A_s^H\right)^{-1} \bs y_{\bs s}.
\end{equation}
In the single-channel setting, $\mathcal A_s \mathcal A_s^H = \mathcal I_{n\times n}$ because of the orthogonality of Fourier exponentials. Thus $\rho_{LS}$ simplifies to 
\begin{equation}\label{v}
\bs{v}=\A ^H \bs{y_s},    
\end{equation}
which is often termed as the zero-filled reconstruction. In this case, the data term involves the comparison of the projection of the reconstructed image and the zero-filled reconstructions. 
\subsubsection{Multichannel case with $\mathbf C = \sigma^2 \mathbf I$: }
  
  In the multichannel setting,  $\mathcal A_s \mathcal A_s^H \neq \mathcal I$, and hence the least-square solution is not equal to the zero-filled reconstruction $\bs v$.  Many deep learning  algorithms use the zero-filled image $\bs v$ as the input to the deep networks rather than its scaled version $\bs u = \sigma^2 \bs v$; we now express the divergence term in \eqref{ensure} in terms of $\bs v$. We consider a network $g_{\phi}$, whose input is $\bs v$ rather than $\bs u$: $g_{\phi}(\bs v) = f_{\phi}(\bs u)$. We thus have 
 \begin{equation}
     \nabla_{\bs{u_i}} \cdot  f_\Phi(\bs{u_i}) = \nabla_{\bs{u_i}} \cdot  g_\Phi(\bs{v_i}) = \nabla_{\bs{v_i}} \cdot  g_\Phi(\bs{v_i}) ~\underbrace{d\bs v_i/d\bs u_i}_{\sigma^2}
 \end{equation}
In this case, the ENSURE loss in \eqref{ensure} can be rewritten as 
\begin{equation}
\label{ensurev}
\text{ENSURE}= \frac{1}{N}\sum_{i=1}^{N} \left(\underbrace{\|\mbf{ R_{\bs s_i}} ( \rec_i -\bs \rho_{\rm{LS},i})  \|_2^2}_{\rm data term} + 2\sigma^2\underbrace{\nabla_{\bs{v_i}} \cdot  g_\Phi(\bs{v_i})}_{\rm divergence}\right) 
\end{equation}
Here, we ignored the constant term $\kappa$. We note that this form is similar to \eqref{sure}, especially when $\mathbf R_s= \mathbf I$.\\\\
We note that the divergence term in \eqref{v} is scaled by $\sigma^2$. The divergence serves as a regularization term on the deep network, making the learned network more robust to input noise/perturbations. The balance between the data consistency term and the divergence term in \eqref{ensurev} is controlled by $2\sigma^2$; more regularization is applied when the input is more noisy. When $\sigma\rightarrow 0$, the divergence term disappears, when the ENSURE based approach simplifies to a loss-function that only involves k-space losses. In this work, we do not consider any additional tuning parameter to the divergence term; our loss function is specified by \eqref{ensurev}.

\subsection{Implementation details}
We compute the least-squares solution as
\begin{eqnarray}\label{cgsense}
\rho_{{\rm LS}} &=& \argmin_{\bs \rho} \|\mathcal A_s \bs \rho - \bs y_s\|^2 + \lambda \|\bs \rho\|^2\\
&=& (\mathcal A_s^H\mathcal A_s + \lambda \mathcal I)^{-1}\mathcal A_s^H \bs y_s
\end{eqnarray} 
with $\lambda \rightarrow 0$. When $\lambda=0$, the above equation can be shown to be equivalent to $\mathcal A_s ^H \left(\mathcal A_s \mathcal A_s^H\right)^{-1} \bs y_{\bs s}$. In the parallel MRI setting, the least-square solution $\rho_{{\rm LS},s} $ in \eqref{ls} is often called the SENSE \cite{Pruessmann1999} solution. We solve \eqref{cgsense} using the conjugate gradients algorithm.

We compute the projection operator $\mbf P_{\bs s}$ as $\mbf P_{\bs s}= \left(\mathcal A_s^H \mathcal A_s\right)^{\dagger} \left(\mathcal A_s ^H \mathcal A_s\right)$ in \eqref{qmse} and \eqref{ensure}. We note that the projection can be obtained by replacing $\bs y_s$ in \eqref{cgsense} with $\bs y_s=\mathcal A_s \bs \rho$. We thus have:

\begin{eqnarray}
\mbf P_{\bs s}\bs \rho = \argmin_{\bs \zeta} \|\mathcal A_s \bs \zeta - \mathcal A_s{\bs  \rho}\|^2 + \lambda \|\bs \zeta\|^2
\end{eqnarray}
with $\lambda\rightarrow 0$; in the MRI setting, this corresponds to the SENSE recovery from $\mathcal A_s \bs e$. When the weighting is also included, we compute 
\begin{eqnarray}\label{projimplementation}
\mbf R_{\bs s}\bs e = \argmin_{\bs \zeta} \|\mathbf D_s\mathcal A_s\left(\bs \zeta - \bs e\right)\|^2 + \lambda \|\bs \zeta\|^2,
\end{eqnarray}
 where $\mathbf D_s =  {\rm diag}\left(\bs d_s^{-\frac{1}{2}}\right)$, where $\bs d_s$ corresponds to the density at the sampling locations. 
 
 Combining \eqref{cgsense} and \eqref{projimplementation}, the computational structure of the data term is shown in Fig.~\ref{fig:sketch}(a). In particular, we compare the recovered images with their CG-SENSE reconstructions specified by $\rho_{{\rm LS}} $. Note that the CG-SENSE solutions may have residual aliasing components. To make the data consistency term insensitive to these errors, the error $\bs e=\widehat{\bs \rho}-\bs \rho_{\rm LS}$ is projected to the range space of $\mathcal A_s$. An additional weighting is used to compensate for the non-uniform density of the sampling patterns in k-space, which is implemented as in \eqref{projimplementation}.

The divergence term in \eqref{ensure} is computed using Monte-Carlo simulations \cite{mcsure} as in Fig.~\ref{fig:sketch}(b). Specifically, noise perturbations are added to the input to the network, and the corresponding perturbations in the output of the network are estimated. The divergence term is approximated as 
\begin{equation*}
	\text{div}_u(f_\Phi(u)) = \lim_{\epsilon \to 0} ~\frac{1}{\epsilon}~\mathbb E_n \left[ \bs n^T\left(f_\Phi(\bs u_s+\epsilon~ \bs n) - f_\Phi(\bs u_s)\right) \right ].
\end{equation*}
Here $n$ is the standard Gaussian. As in \cite{mcsure}, we observe that the expectation can be approximated by a Monte-Carlo approach involving different noise realizations. In practice, one noise realization per image per epoch is sufficient to obtain a good approximation as observed in \cite{mcsure}. Note that the use of the data term alone will result in over-fitting, similar to observations in DIP methods as the number of epochs increases. The divergence term may be viewed as a network regularization, which serves to minimize noise amplification. 

In this work, we estimate the $\mathbf W$ matrix by evaluating multiple sampling matrices and by computing their average. 

In this work, we have assumed the Cartesian sampling mask, where each entry is a Bernoulli random variable whose probability at each spatial frequency location to be given by a Gaussian function; the probability of the samples is higher in the k-space center and lower at higher frequencies. We do not fully sample the k-space center. Algorithm~I summarizes the steps of the proposed ENSURE loss for performing unsupervised training.

\begin{algorithm}
	\caption{ENSURE loss function calculation.} 
 
	\begin{algorithmic}[1]
		\renewcommand{\algorithmicrequire}{\textbf{Input:}}
		\renewcommand{\algorithmicensure}{\textbf{Output:}}
		\REQUIRE $\bs y_{\bs s}$, $\A$
		\ENSURE  ENSURE loss		
		\STATE Estimate the zero-filled reconstruction $u=\A^H \bs y_{\bs s}$.
  \STATE Estimate the least square solution $\rho_{LS}$ by solving Eq.~\eqref{ls}.
  \STATE Estimate the network prediction $\rec$ using $f_{\phi}$.
  \STATE find the error as $e=\rec-\rho_{LS}$.
  \STATE Solve Eq.~\eqref{projimplementation} to apply weighted projected $\mathcal R_s$ on $e$. 
  \STATE Take the L2-norm of $\mathcal R_s e$ to get the data-term.
  \STATE Find divergence term of Eq.~\eqref{ensurev} as in Fig.\ref{fig:sketch}~(b).
  \STATE Estimate the ENSURE loss as the average of data-term and scaled divergence term as in Eq.~\eqref{ensurev}.  
  \end{algorithmic} 

\end{algorithm}

\subsection{Deep network architectures}

We used the model-based deep learning architecture, as shown in Fig.~\ref{fig:network} for the initial comparisons. The CNN had five layers, each with $3\times3$ convolution, batch normalization, and ReLU non-linearity. Fig.~\ref{fig:network} also shows the number of feature maps on top of each layer.  The network had shared parameters in three unrolling steps. We trained the network for a total of 50 epochs with a fixed learning rate of $10^{-3}$. The data-consistency (DC) step utilized the complex data,  whereas the CNN used the real and imaginary components of the complex data as channels.

\section{Results}

\begin{figure} 
	\centering
	\includegraphics[width=.99\linewidth]{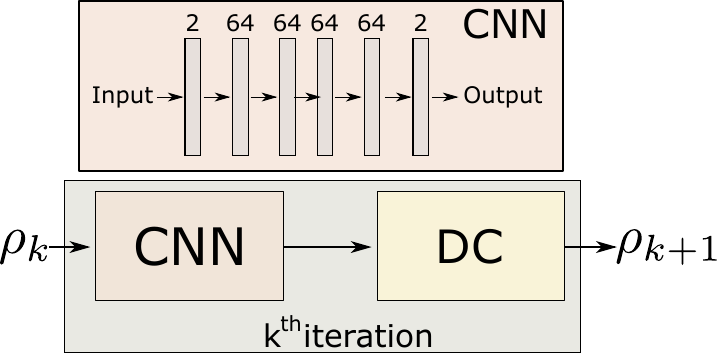}   
	\caption{The specific network architecture used in the experiments. The CNN is a five-layer model that takes complex data as input. Each of the middle layers has 64 feature maps. It concatenates real and imaginary components as channels in the first layer. Similarly, the last layer output is converted to the complex type before it enters the DC step.}
	\label{fig:network}
\end{figure}


\begin{table} \centering
	\caption{Quantitative comparison ofpeak signal-to-noise ratio (PSNR)~(dB) and structural similarity index (SSIM) values. Each value shows mean~$\pm$~ std.   }
	\label{tab:results}
		\begin{tabular}{lccc}   \toprule
                   	   &\multicolumn{3}{c}{Peak Signal to Noise Ratio (PSNR)} \\ \midrule
         acceleration  &        4x        &4x, $\sigma=0.03$&   6x            \\ \midrule
		E-to-E-Sup & $37.71 \pm0.97 $ & $ 35.34\pm0.70 $ & $35.39 \pm 0.92$\\  
	    L-by-L-Sup  & $36.64 \pm1.07$  & $ 35.29 \pm1.29$ & $35.26 \pm0.99$\\ 
		\midrule
		GSURE          & $ 36.80\pm1.04 $ & $ 31.10\pm1.40 $ & $33.68 \pm 1.04$\\
    	K-MSE          & $ 35.29\pm1.12 $ & $ 31.67\pm0.76 $ & $33.16 \pm 1.04$\\
		SSDU           & $ 35.46\pm1.03 $ & $ 33.48\pm0.71 $ & $33.92 \pm 1.15$\\
		LDAMP-SURE     & $ 34.97\pm1.23 $ & $ 33.56\pm1.39 $ & $32.87 \pm 1.43$\\
		ENSURE         & $ 37.36\pm0.92 $ & $ 34.13\pm0.90 $ & $34.98 \pm 0.93$\\ \midrule
                 	   &\multicolumn{3}{c}{Structural Similarity Index (SSIM)} \\ \midrule
		E-to-E-Sup & $ 0.97\pm0.00 $ & $ 0.95\pm0.01 $ & $ 0.94\pm0.01 $ \\ 
   	    L-by-L- Sup   & $ 0.95\pm0.00 $ & $ 0.95\pm0.00 $ & $ 0.94\pm0.00 $ \\ \midrule
		GSURE          & $ 0.93\pm0.01 $ & $ 0.91\pm0.01 $ & $ 0.92\pm0.02 $ \\
		K-MSE          & $ 0.95\pm0.01 $ & $ 0.86\pm0.02 $ & $ 0.93\pm0.01 $ \\
		SSDU           & $ 0.93\pm0.02 $ & $ 0.88\pm0.04 $ & $ 0.92\pm0.01 $ \\
		LDAMP-SURE     & $ 0.96\pm0.00 $ & $ 0.94\pm0.01 $ & $ 0.93\pm0.01 $ \\
		ENSURE         & $ 0.97\pm0.00 $ & $ 0.93\pm0.01 $ & $ 0.94\pm0.01 $ \\ 
		\bottomrule		
	\end{tabular}
\end{table}

We consider publicly available~\cite{modl} parallel MRI brain data acquired using a 3-D~T2~CUBE sequence with Cartesian readouts using a $12$-channel head coil at the University of Iowa on a 3T GE MR750w scanner. The Institutional Review Board at the University of Iowa approved the data acquisition, and written consent was obtained from the subjects. The matrix dimensions were $256\times232\times 208$ with a $1$~mm isotropic resolution.  Fully sampled multi-channel brain images of nine volunteers were collected, out of which data from five subjects were used for training. The data from two subjects were used for testing and the remaining two for validation.

The validation data was used to terminate the training process. We utilized the ESPiRiT algorithm~\cite{espirit2014} to estimate the coil sensitivity maps.  

We also used a subset of the fast-MRI \cite{fastMRI} knee dataset, consisting of 100 subjects, to demonstrate the utility of ENSURE on multiple unrolled architectures. The data from 43 subjects were used for training, two for validation, and the rest for testing.

\subsection{Comparison with state-of-the-art methods}

We compare the proposed ENSURE approach with the following supervised and unsupervised methods. 

\noindent\textbf{Two supervised learning algorithms} E-to-E-Sup and L-by-L-Sup: The E-to-E-Sup approach refers to the end-to-end supervised training approach \cite{modl}. To make it consistent with the rest of the experiments, we use multiple sampling patterns during training. The L-by-L-Sup uses a similar acquisition and reconstruction architecture but relies on the supervised training of the CNNs in each layer. Both approaches use the loss function in \eqref{sup_mse} for training. We note that the number of free parameters in L-by-L-Sup is threefold higher than E-to-E-Sup, where the CNN parameters are shared across iterations.   

\noindent\textbf{GSURE-based unsupervised learning:} The GSURE~ loss, which is an unbiased estimate of projected MSE in~\eqref{pmse}, was originally introduced for regularization parameter selection.  We consider the direct use of GSURE to train deep image reconstruction algorithms. Here, the sampling operator is assumed to be the same for all training images as considered in \cite{ldampSURE}.

\noindent\textbf{Unsupervised learning with different sampling patterns:} We compare the proposed scheme against K-MSE, SSDU~\cite{ssdu}, and LDAMP-SURE~\cite{ldampSURE}. The measurement operators are randomly chosen for each image in all of these approaches, but are assumed to be fixed for each image; we do not vary them during iterations. The K-MSE~\eqref{kmse} loss function involves the square of the error between measured and predicted k-space. 
As  discussed previously, the K-MSE approach is vulnerable to over-fitting to noise. To minimize over-fitting, SSDU partitions the measured k-space samples into two disjoint groups: one group is used to reconstruct the images, while the second is used to evaluate the loss function. The 80-20 partition of k-space offered the best PSNR in our experimental setup. The LDAMP-SURE algorithm relies on layer-by-layer training. The five-layer CNN at each layer is independently trained as a denoiser using the SURE loss function.

We consider the recovery from 2-D random sampled multi-channel data in Table~\ref{tab:results}. The PSNR and SSIM~\cite{ssim} for two different acceleration factors of 4x and 6x are reported in the first and last columns. To study the impact of measurement noise on the algorithms, we also consider a setting with additional Gaussian noise of standard deviation~$\sigma=0.03$ added to the undersampled k-space data at 4x acceleration; these results are shown in the second column of Table~\ref{tab:results}. The top two rows consider the supervised setting, while the remaining rows denote the unsupervised methods discussed above. 

Fig.~\ref{fig:4x} visually compares the reconstruction quality of different unsupervised learning techniques with E-2-E-Sup learning at the 4x acceleration factor. In this low-noise setting, K-MSE offers nearly the same reconstruction quality as SSDU, which is also seen in Table \ref{tab:results}. The zoomed region shows that LDAMP-SURE, K-MSE, and SSDU have visible artifacts in the reconstructions compared to GSURE and ENSURE. The image quality of GSURE and ENSURE are comparable in this setting. 

Fig~\ref{fig:6x} shows reconstruction results at a higher acceleration factor of 6x. We note that the performance of the ENSURE approach is comparable to that of the supervised approach at this high acceleration rate. By contrast, we observe that GSURE, K-MSE, and LDAMP-SURE exhibit more noise-like artifacts along with blurring. A green rectangle in the zoomed images shows a hallucinated  feature in the LDAMP-SURE because of layer-by-layer training of individual networks instead of end-to-end training in the proposed ENSURE approach. The SSDU approach offers less-noisy reconstructions compared to these approaches. However, the PSNR is around 0.5 dB lower than ENSURE. An arrow in the zoomed cerebellum region points to a feature that is well preserved by ENSURE while blurred by SSDU.

\subsection{Performance in high-noise setting}

As described previously, the divergence serves as a regularization term on the deep network, making the learned network more robust to input noise/perturbations. In the low noise setting ($\sigma\rightarrow 0$), the weight for the divergence term in \eqref{ensurev} tends to zero. In this case, ENSURE simplifies to the DC term, which is a weighted version of the K-MSE loss. The difference of ENSURE with this simple strategy is higher at high noise settings, which might happen at low-field or FLAIR acquisitions.  To study the benefit of the proposed formulation in high noise settings, Gaussian noise of standard deviation $\sigma=0.03$ was added to the measured k-space data at 4x acceleration. 
The comparison of the methods is shown in Fig. \ref{fig:4xnoise}, which can also be seen in the second column of Table~\ref{tab:results}. In this setting, we observe that E-to-E-Sup and L-to-L-Sup have nearly the same mean PSNR values. We observe that the LDAMP-SURE approach performs better than the GSURE approach in this high noise setting, as reported in \cite{ldampSURE}. In particular, the projected MSE is a poor approximation for the MSE, which translates to poor reconstructions in the presence of noise. The improved performance of SSDU over KMSE shows that the partitioning strategy used in SSDU can reduce the over-fitting to noise. The ENSURE approach offers the best PSNR out of the unsupervised strategies and is comparable in PSNR and image quality to E-to-E-Sup. We observe that ENSURE is robust to increased noise variance when compared to other algorithms that do not use the divergence term. While GSURE also uses the divergence term, the data consistency term is very different from ENSURE, which causes an imbalance between the two terms; we attribute the lower performance of GSURE in this setting to the imbalance.

\begin{figure*} 
	\input{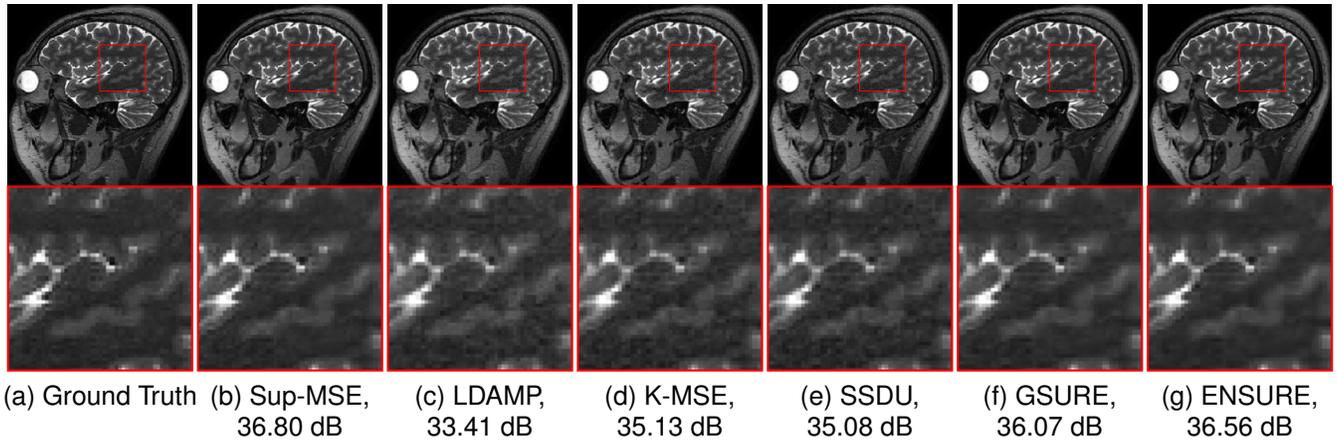}
	\caption{This figure compares reconstruction results at 4x acceleration on a test slice. We can observe from the zoomed region that the proposed ENSURE approach results in comparable reconstruction quality to that of supervised training with MSE (Sup-MSE). Here, LDAMP refers to the LDAMP-SURE algorithm. }
	\label{fig:4x}
\end{figure*}
\begin{figure*} 
	\input{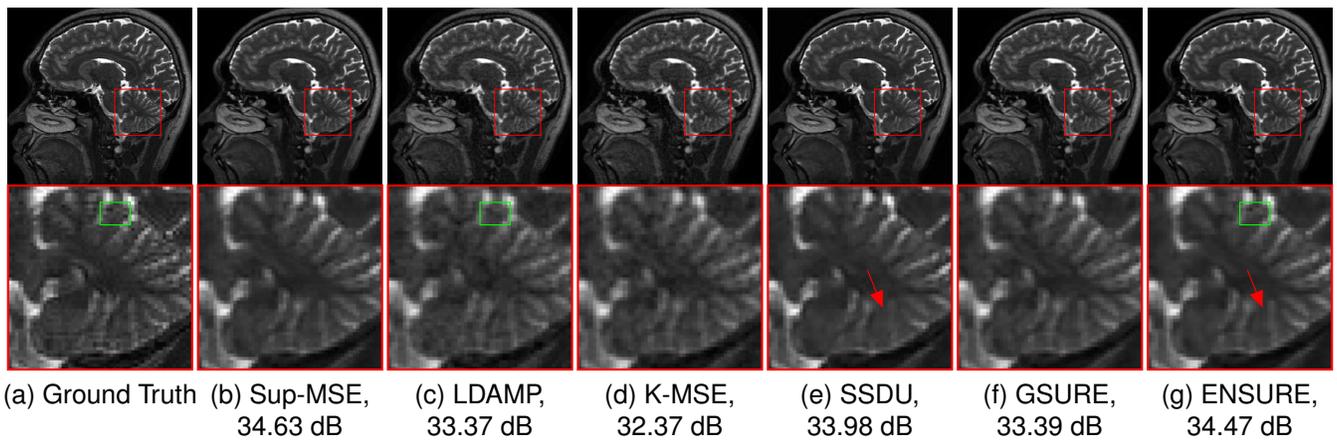}
	\caption{Reconstruction results at a higher acceleration factor of 6x. We observe that, even at this high acceleration, the proposed ENSURE approach results in similar quantitative and qualitative reconstruction quality to that of supervised MSE training. The arrow in the zoomed region points to a feature well captured using ENSURE as compared to SSDU. }
	\label{fig:6x}
\end{figure*}
\begin{figure*} 
	\input{fig_compare_2d_4x_noise_err.tex}
	\caption{Reconstruction results at 4x acceleration with added Gaussian noise of $\sigma=0.03$ in the measurements. The bottom row shows the corresponding error maps. The magnitude of error images is multiplied by four for visualization purposes. }
	\label{fig:4xnoise}
\end{figure*}

\subsection{Impact of different loss terms: ablation study}
\label{impact}
The ENSURE approach has two key differences from the GSURE framework: (a) the use of an ensemble of sampling patterns, and (b) the use of  additional weighting within the DC term. We now study the impact of these components on performance. 

\subsubsection{Need for using an ensemble of measurement operators}
We compare the GSURE framework, which uses a single sampling pattern for all the images, with the ENSURE formulation, which uses multiple sampling operators in Figure~\ref{fig:gsure_ens}. A one-dimensional Cartesian sampling at 2x acceleration with an additional Gaussian noise of $\sigma=0.02$ was used. We note that aliasing artifacts and blurring are visible in the GSURE approach. In contrast, the ENSURE approach offers comparable performance to the E-to-E-Sup-MSE. A major source of the improved performance is the use of the ensemble of measurement operators, which makes ENSURE closely approximate the true MSE. By contrast, the GSURE approach is only approximating the projected MSE, which is inferior in the undersampled setting, as discussed in \cite{ldampSURE}.

\begin{figure*} 
	\input{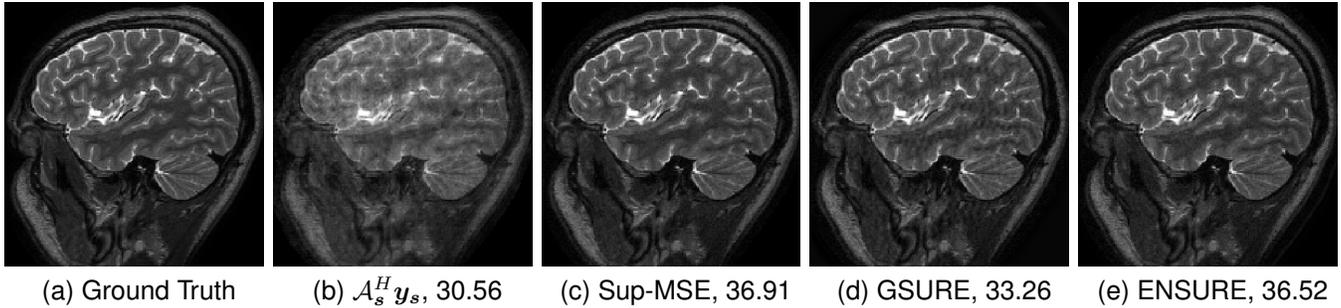}
	\caption{This example demonstrates the benefit of sampling an image with multiple measurement operators in the one-dimensional sampling case at a 2X acceleration factor with added Gaussian noise of $\sigma=0.02$. The numbers in the sub-captions show PSNR (dB) values.}
	\label{fig:gsure_ens}
\end{figure*}

\subsubsection{Impact of the weighting}

\begin{figure*} 
	\input{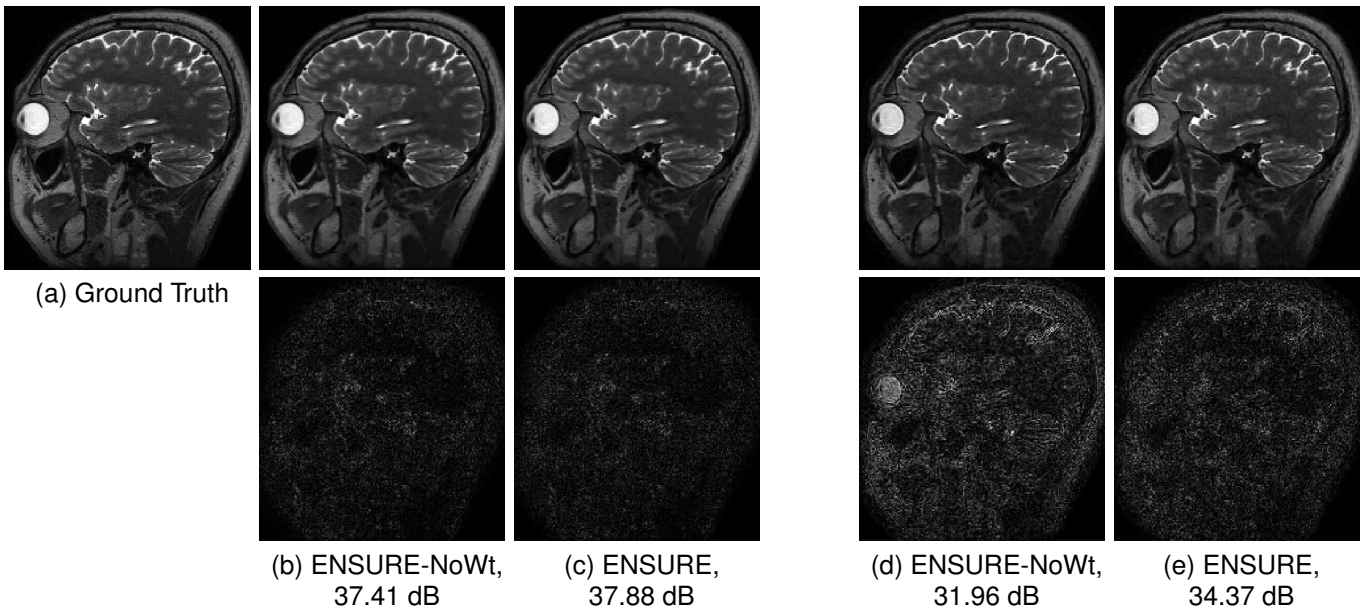}
	\caption{This figure shows the benefit of the proposed weighting strategy. It visually compares the reconstruction quality of ENSURE with and without the proposed weighting scheme. The magnitude of error images is multiplied by four for visualization purposes. (b) and (c) refers to the case when no simulated noise is added to the measurements. (d) and (e) refers to the experimental setup when we added Gaussian noise of $\sigma=0.03$ to the measurements. }
	\label{fig:weighting}
\end{figure*}

This experiment compares the impact of the proposed weighting strategy in~\eqref{wpmse} at 4x acceleration using a 2D random variable-density sampling pattern. 
We trained two networks with experimental setups that are similar except for the loss function. When we set the weights $\mbf W$ to identity in~\eqref{ensure}, it results in a network trained with the ENSURE loss without weighting. The reconstruction quality and error maps in Fig.~\ref{fig:weighting} demonstrate that the proposed weighting scheme improves the reconstruction quality significantly. 

\subsection{Applicability to other deep image recovery algorithms}
We demonstrate the applicability of the proposed loss function to other deep image recovery algorithms at 2x acceleration in Table \ref{fastMRIval}; the reconstructed images are not shown in this case because there are no substantial differences between them. Here, we use a subset of 100 datasets from fastMRI \cite{fastMRI}. The noise variance was estimated from one fully-sampled dataset, which was used for the rest of the datasets.  The first four columns in Table \ref{fastMRIval} correspond to different algorithms trained using fully sampled ground truth and ENSURE, respectively. A uniform density sampling pattern with the fully sampled center for coil sensitivity estimation, provided by fastMRI challenge~\cite{fastMRI}, was used in these experiments. In each column, we report the PSNR (top) and the SSIM (bottom) values. The comparisons show that all of the algorithms (ADMMnet, FISTAnet, and MoDL) offer somewhat similar performance on this dataset when trained with MSE. The results also show that ENSURE is capable of offering comparable reconstructions for these algorithms.  

We also study the impact of ENSURE under different sampling conditions in the last three columns on Table \ref{fastMRIval}. MoDL-MS refers to ENSURE implementation, where a different random undersampled mask is used at each epoch. We note that such an approach is not feasible in settings where only the undersampled data is available. This experiment is to determine the impact of mask shuffling on image quality. We note that the results in this setting are compared to MoDL without mask shuffling. This shows that the approximation of the expectation over the sampling pattern in \eqref{wpmse} is accurate, even when only one sampling pattern per image is considered (MoDL setting, without mask, shuffling). The last two columns, SENSE-VD and MoDL-VD, consider 2x variable-density sampling, using masks similar to the brain experiments. We note that the higher density in the center of k-space translates to improved performance. However, the gap between MSE and ENSURE remains almost the same, indicating that the specific choice of sampling patterns have less impact on the ENSURE approximation.

\section{Discussion}

We observe from Table~\ref{tab:results}, as well as the figures, that the proposed unsupervised learning using the ENSURE loss function offers PSNR and SSIM values that are comparable to those of E-to-E-Sup. The experiments in Section \ref{impact} show that the ensemble of measurement operators and the weighted loss function enabled the ENSURE loss to closely approximate the true MSE when fully sampled and noise-free images are not available. 

We note that the performance of the L-by-L-Sup approach, which has threefold more parameters, is only marginally worse than the E-to-E-Sup approach in all the settings. We note that the LDAMP-SURE approach, which relies on layer-by-layer training, is worse than ENSURE. This deterioration in performance may be attributed to the non-Gaussianity of the alias artifacts at each iteration, which is the key assumption used in LDAMP-SURE. We observe that the K-MSE scheme offers low performance, especially at high-noise and highly accelerated settings. This is caused by the over-fitting of the network to the noise within the measurements. The partitioning strategy used in SSDU is observed to offer improved performance over K-MSE, especially in the high-noise setting. Similarly, the performance improvement of ENSURE over GSURE is more significant in the high-noise and highly accelerated setting. We trained the networks for each setting independently. The experiments with additional synthetic noise involved both training and testing the networks at the high noise setting. We note that all model-based algorithms (e.g. SENSE that combines information from multiple channels and compressed sensing) simultaneously perform reconstruction and denoising; the reconstructed signal is not data-consistent in any of these cases. 

The ENSURE and GSURE techniques depend on the estimation of noise variance in the measured data. In the experiments in Table \ref{fastMRIval}, we estimate the noise variance from one fully sampled complex coil combined image. Since the image is complex, the noise will still be Gaussian distributed. We threshold the reconstructed image to identify the image support. The variance of the image in the background region is used as the noise variance.
 We estimate the noise variance for LDAMP-SURE \cite{koreanReconCVPR2019,ldampSURE} as suggested in~\cite{koreanReconCVPR2019}. 

We used the network architecture in Fig.~\ref{fig:network} for the comparisons of loss functions. However, any direct-inversion or model-based network architecture can use the proposed ENSURE loss function. Unlike LDAMP-SURE, our proposed method does not depend on the unrolling of any specific algorithm. 

The SSDU method~\cite{ssdu} suggests partitioning the measured k-space such that 60\% is used in the DC step and 40\% in the loss function estimation. With our dataset and network architecture, we experimented with several different partitions of the k-space, including the suggested 60-40 split, and found that an 80-20 division performed best.  

The main focus of this work is on how to train a network without fully sampled reference images. We note that there are several practical considerations, including dependence of the trained network on the specific field of view and contrasts, which are non-trivial issues. As shown in earlier work \cite{modl,modlmussels}, the training strategy may be modified to make the network relatively less sensitive to these changes. However, these problems are shared between supervised and unsupervised methods. To keep the paper focused, we do not address these aspects in this work. In this work, we train different networks for different undersampling ratios. As shown in \cite{modl,modlmussels}, the proposed work may be extended to make the trained network insensitive to the specific sampling ratio. However, we leave this potential extension for future work to keep the paper brief.  

In this work, we assume that the average weight matrix $\mathbf W$ is invertible. In the single channel setting, this will not be true if some Fourier samples are not included or acquired by any of the sampling patterns. In this case, it will be impossible to realize an unbiased estimate for the true MSE. One may come up with a projected MSE onto the observed samples. We leave this potential extension for future work, with the focus on keeping the paper brief.

We note that there is a gap between the supervised MSE training and the ENSURE results. We note that ENSURE is only an unbiased estimate for the MSE, which is an approximation. The approximation becomes exact only when the expectation is taken over the sampling patterns. In our implementation, we assume that each image is sampled with a specific sampling pattern that is fixed during the training process. We also note that the gap in performance of ENSURE and MSE is comparatively higher in the fastMRI setting than in our earlier experiments. We attribute this to the larger variability in image quality between the datasets in fastMRI (likely acquired on different scanners and coils) and the inaccuracy in the noise estimation. We will consider the estimation of noise variance for each subject in the future to close the gap.

The main focus of this paper is to establish the theory and validate it with datasets where fully sampled images are available. We note that this approach is most useful in dynamic MRI and high-resolution static MRI, where the acquisition of fully sampled ground-truth images is not feasible. We will consider the application of ENSURE to these settings in our future work; this is beyond the scope of this preliminary publication. 

We note that SURE-based approaches have been recently used in MRI \cite{vineet}.  The main focus of the above paper is on uncertainty quantification using variational auto-encoders, while the authors also use SURE to approximate the MSE loss. By contrast, the main focus of this work is to introduce the ENSURE metric and to elaborately compare it with MSE and state-of-the-art unsupervised methods under different undersampling and noise settings. 
In addition, the authors approximate the MSE by SURE, assuming the density-weighted undersampling noise to be Gaussian \cite{vineet}; while the mean of the density-weighted noise term is shown to be zero in \cite{vineet}, it is not sufficient to guarantee that the undersampling-induced noise is Gaussian. By contrast, we define a new metric in \eqref{wpmse} which is the average of the errors over datasets and sampling patterns, which is well approximated by the ENSURE metric.

\begin{table*}[h!]
    \centering
    \begin{tabular}{|c|c|c|c|c|c|c|c|}
    \hline
       & SENSE: U   &  ADMM-net: U & FISTA-net: U & MoDL: U & MoDL-MS & MoDL VD1 & MoDL-VD2 \\
      \hline 
      MSE & 38.95 ± 4.2 & 42.09 ± 3.8 & 42.04 ± 3.8 & 42.30 ± 3.8 & 42.38 ± 3.9 & 43.90 ± 3.7 & 44.45 ± 3.9\\
          & 0.965 ± 0.02 & 0.983 ± 0.01 & 0.981 ± 0.01 & 0.985 ± 0.01 & 0.985 ± 0.01 & 0.99 ± 0.01 & 0.990 ± 0.01
\\
      \hline
      ENSURE & & 39.99 ± 4.0 & 40.24 ± 3.8 & 40.75 ± 3.7 & 40.67 ± 3.9 &41.69 ± 3.9& 42.17 ± 4.0\\
      && 0.972 ± 0.02 & 0.972 ± 0.02 & 0.973 ± 0.02 & 0.973 ± 0.02 &0.98±0.01& 0.983 ± 0.01\\
      \hline
    \end{tabular}
    \caption{ Quantitative results on fastMRI dataset. The values represent the average of PSNR(dB) $\pm$ standard deviations on the test dataset.}
    \label{fastMRIval}
\end{table*}
\section{Conclusions}


We introduced a novel loss function for the unsupervised training of deep-learning-based image reconstruction algorithms when fully sampled training data is not available. The proposed approach is the extension of the generalized SURE (GSURE) approach. The key distinction of the framework from GSURE is the assumption that different images are measured by different sampling operators. We evaluate the expectation of the GSURE losses over the sampling patterns to obtain the ENSURE loss function, which is an unbiased estimate for the true MSE. Our theoretical results show that the proposed ENSURE loss function closely approximates the true MSE without requiring the fully sampled images. More specifically, the use of the different sampling operators enables us to obtain a better approximation of the loss than GSURE, which approximates the projected MSE; the MSE of the projections of the images to the range space of the measurement operator is a poor approximation of the true MSE, especially in highly undersampled settings. The experiments confirm that the deep learning networks trained using the ENSURE approach  without fully sampled training data closely resemble the networks trained using a supervised loss function. While our focus in this work was on an MRI reconstruction problem, this approach is broadly applicable to general inverse problems, including deblurring and tomography. 


\appendices

\section*{Appendix}

From (1), we note that $
\bs{u}=\A ^H \mathbf C^{-1} \bs{y_s}$ is Gaussian distributed with mean $\A^H\mathbf C^{-1} \A \bs \rho $ and covariance matrix $\A^H\mathbf C^{-1}\A$, i.e.,
\begin{eqnarray}
		p(\bs{u}|\bs \rho) &= \mathcal{N}\left( \A^H \mathbf C^{-1} \A \bs \rho , \A^H \mathbf C^{-1}\A\right)\\
		&= q(\bs{u}) ~\exp \left( \bs \rho^T \bs{u} - g(\bs\rho) \right ). 
\end{eqnarray}
Here, $q(\bs u)$ is independent of $\bs\rho$. Similarly, $g(\bs \rho)$ is dependent on $\rho$, which is not often available, and is independent of $\bs u$ :
\begin{eqnarray}
		q(\bs{u}) &=& K \exp\left[-\frac{1}{2}\bs {u}^H \left[\A^H\mathbf C^{-1} \A \right]^{\dagger}\bs {u} \right] \\
		g(\bs\rho) &=&\frac{1}{2} \bs\rho^H \A^H \mathbf C^{-1} \A \bs\rho
\end{eqnarray}

Here, $K$ is a constant. We consider the third term in \eqref{expanded}. Denoting $\mbf D=\mbf{ W_s}^T\mbf{ W_s}$, we obtain
\begin{eqnarray*}
\Eus \left[ \rec^T \mbf D \bs \rho \right] 
		&=&  \int_{-\infty}^{+\infty}  \rec^T \mbf D \bs \rho \; \underbrace{q(\bs{u})\overbrace{ \exp \bigl (\bs \rho^T\bs{u}-g(\bs\rho)}^{h(\bs{u})} \bigr)}_{p(\bs u)}  d\bs{u}
\end{eqnarray*}
		
Let   $h(\bs{u})= \exp \bigl (\bs \rho^T\bs{u}-g(\bs\rho) \bigr)$, then $\nabla_{\bs u} h\left(\bs u\right) = \bs\rho~ h\left(\bs u\right)$. Substituting for $\bs\rho ~h\left(\bs u\right)$, we obtain
\begin{eqnarray*}\nonumber
\Eus \left[ \rec^T \mbf D \bs \rho \right]  &=& \int_{-\infty}^{+\infty}   \rec^T \mbf D ~q(\bs{u}) \; \nabla_{\bs u} h\left(\bs u\right)d\bs{u}\\\nonumber
		&=& \left\langle q(\bs{u})\mbf D\rec,\, \nabla_{\bs u} h\left(\bs u\right)  \right\rangle\\\nonumber
		&=& - \left\langle h(\bs{u}),  \nabla_{\bs{u} } \cdot \left[ q(\bs{u})\mbf D\rec \right]\right\rangle
\end{eqnarray*}
Expanding the second term in the inner-product using the chain rule, we obtain
\begin{eqnarray*}
\nabla_{\bs{u}}  \cdot \left[ q(\bs{u})\mbf D\rec \right] = q(\bs{u})  \nabla_{\bs{u} } \cdot \mbf D\rec + \mbf D\rec \cdot \nabla_{\bs{u}}q(\bs{u})  \\
= q(\bs{u}) \left(\nabla_{\bs{u} } \cdot \mbf D\rec + \mbf D\rec\cdot \frac{\nabla_{\bs{u}} q(\bs{u})}{q(\bs{u})}\right)\\
= q(\bs{u})\left( \nabla_{\bs{u} } \cdot \mbf D\rec + \mbf D\rec\cdot ~\underbrace{\nabla_{\bs{u}} \ln q(\bs{u})}_{-\bs \rho_{\rm LS}}\right)
\end{eqnarray*}
 In the last step, we used the property
\begin{equation}\label{lsrecon}
-\nabla_{\bs u} \ln q(\bs{u})=   \left[\A^H\mathbf C^{-1} \A \right]^{\dagger} \A^H \mathbf C^{-1} \bs{y_s} = \bs \rho_{\rm LS}
\end{equation}
Because $q(\bs u) h(\bs u) = p(\bs u)$, we obtain
\begin{eqnarray*}
\Eus \left[ \rec^T \mbf D \bs \rho \right] 
		&=& - \mathbb E_{\bs u}\left[ \nabla_{\bs{u}} \cdot \mbf D \rec \right] + \mathbb E_{\bs u}\left[\rec^T \mbf D\bs \rho_{\rm LS}  \right]
\end{eqnarray*}

\bibliographystyle{IEEEtran}
\balance


\end{document}